\documentclass{llncs} 

\usepackage{amssymb}
\usepackage{graphicx}
\usepackage{amsmath}
\usepackage{mathtools}
\usepackage{tikz}

\newcommand{\featdim}{d}
\newcommand{\Nunl}{U}
\newcommand{\Nlab}{L}
\newcommand{\X}{\mathbf{X}  }
\newcommand{\Xe}{\mathbf{X}_e  }
\newcommand{\XeT}{\mathbf{X}_e^{\top}}

\newcommand{\ye}{\begin{bmatrix} \mathbf{y}  \\ \mathbf{y}_u \end{bmatrix}}

\newcommand{\Cb}{\mathcal{C}_{\beta}}

\begin{document}

\title{Implicitly Constrained Semi-Supervised Least Squares Classification}

\author{Jesse H. Krijthe\inst{1,2} \and Marco Loog\inst{1,3}}

\institute{Pattern Recognition Laboratory, Delft University of Technology
\and
Department of Molecular Epidemiology, Leiden University Medical Center
\and
The Image Group, University of Copenhagen
}
\maketitle
\begin{abstract}
We introduce a novel semi-supervised version of the least squares classifier.
This implicitly constrained least squares (ICLS) classifier minimizes the squared loss on the labeled data among the set of parameters implied by all possible labelings of the unlabeled data.
Unlike other discriminative semi-supervised methods, our approach does not introduce explicit additional assumptions into the objective function, but leverages implicit assumptions already present in the choice of the supervised least squares classifier.
We show this approach can be formulated as a quadratic programming problem and its solution can be found using a simple gradient descent procedure. 
We prove that, in a certain way, our method never leads to performance worse than the supervised classifier.
Experimental results corroborate this theoretical result in the multidimensional case on benchmark datasets, also in terms of the error rate.
\end{abstract}

\section{Introduction}
Semi-supervised classification concerns the problem of using additional unlabeled data, aside from only labeled objects considered in supervised learning, to learn a classification function.
The challenge of semi-supervised learning is to incorporate this additional information to improve the classification function over the supervised function.

The goal of this work is to build a semi-supervised version of the least squares classifier that has the property that, at least in expectation, its performance is not worse than supervised least squares classification.
While it may seem like an obvious requirement for any semi-supervised method, current approaches to semi-supervised learning do not have this property. 
In fact, performance can significantly degrade as more unlabeled data is added, as has been shown in \cite{Cozman2006,Cozman2003}, among others.
This makes it difficult to apply these methods in practice, especially when there is a small amount of labeled data to identify possible reduction in performance.
A useful property of any semi-supervised learning procedure would therefore be that its performance does not degrade as we add more unlabeled data.

We present a novel approach to semi-supervised learning for the least squares classifier that we will refer to as implicitly constrained least squares classification (ICLS). 
ICLS leverages implicit assumptions present in the supervised least squares classifier to construct a semi-supervised version. This is done by minimizing the supervised loss function subject to the constraint that the solution has to correspond to the solution of the least squares classifier for some labeling of the unlabeled objects.
Through this formulation, we exploit constraints inherent in the choice of the supervised classifier whereas current state-of-the-art semi-supervised learning approaches typically rely on imposing additional extraneous, and possibly incorrect, assumptions \cite{Seeger2001,Singh2008}.

This work considers a semi-supervised version of the supervised least squares classifier, in which classes are encoded as numerical outputs after which a linear regression model is applied (see section \ref{section:leastsquares}). By placing a threshold on the output of this model, one can use it to predict class labels. 
In a different neural network formulation, this classifier is also known as Adaline \cite{Widrow1960}.
There are several reasons why the least squares classifier is a particularly interesting classifier to study: 
First of all, the least squares classifier is a discriminative classifier. 
Some have claimed semi-supervised learning without additional assumptions is impossible for discriminative classifiers \cite{Seeger2001,Singh2008}. 
Our results show this may not strictly hold. 
Secondly, as we will show in section \ref{section:icls}, the closed-form solution for the supervised least squares classifier allows us to study its theoretical properties.
Moreover, using the closed-form solution we can rewrite our semi-supervised approach as a quadratic programming problem, which can be solved through a simple gradient descent with boundary constraints. 
Lastly, least squares classification is a useful and adaptable classification technique  allowing for straightforward use of, for instance, regularization, sparsity penalties or kernelization \cite{Hastie2001,Rifkin2003,Poggio2003,Tibshirani1996}. 
Using these formulations, it has been shown to be competitive with state-of-the-art methods based on loss functions other than the squared loss \cite{Rifkin2003} as well as computationally efficient on large datasets \cite{Bottou2010}.


The main contributions of this paper are
\begin{itemize}
  \item A novel convex formulation for robust semi-supervised learning using squared loss (Equation \eqref{icls})
  \item A proof that this procedure never reduces performance in terms of the squared loss for the 1-dimensional case (Theorem \ref{theorem:1d})
  \item An empirical evaluation of the properties of this classifier (Section \ref{section:empiricalresults})
\end{itemize}

We start with a discussion of related work after which we introduce our semi-supervised version of the least squares classifier. In Sections \ref{section:theoreticalresults} and \ref{section:empiricalresults}, we study the non-degradation property of this method both theoretically and by considering the method's behaviour on benchmark datasets. In the final sections we discuss the results and conclude.

\section{Related Work} 
\label{section:relatedwork}
Many diverse semi-supervised learning techniques have been proposed \cite{Chapelle2006,Zhu2009}. Most of these techniques rely on introducing useful assumptions that link information about the distribution of the features $P_X$ to the posterior of the classes $P_{Y|X}$. Some have argued unlabeled data can \emph{only} help if $P_X$ and $P_{Y|X}$ are somehow linked  through one of these assumptions \cite{Singh2008}.
While these methods have proven successful in particular applications, such as document classification \cite{Nigam2000}, it has also been observed that these techniques may give performance worse than their supervised counterparts, see \cite{Cozman2006,Cozman2003}, among others. 
In these cases, disregarding the unlabeled data would lead to better performance. 

The method considered in our work is different from most previous work in semi-supervised learning in that it is inherently robust against this decrease in performance. We show that one does not need extrinsic assumptions for semi-supervised learning to work. In fact, such assumptions may actually be at the root of the problem: clearly if such an additional assumption is correct, the semi-supervised classifier can gain from it, but if the assumption is incorrect, degraded performance may ensue.  What we will leverage in our approach are the implicit assumptions that are, in a sense, intrinsic to the supervised least squares classifier. This work is in line with the proposal of \cite{Loog2014a,Loog2014b} which set out to improve likelihood based classifiers in a similar way. Our approach, however, does not rely on explicitly formulating the intrinsic constraints on the estimated parameters. Moreover, our approach allows for theoretical analysis of the non-deterioration of the performance of the procedure.

Another attempt to construct a robust semi-supervised version of a supervised classifier has been made in \cite{Li2015}, which introduces the safe semi-supervised support vector machine (S4VM). 
This method is an extension of semi-supervised SVM \cite{Bennett1998} which constructs a set of low-density decision boundaries with the help of the additional unlabeled data, and chooses the decision boundary, which, even in the worst-case, gives the highest gain in performance over the supervised solution. 
If the low-density assumption holds, it can be proven this procedure increases classification accuracy over the supervised solution. 
The main difference with the method considered in this paper, however, is that we make no such additional assumptions. We show that even without such assumptions, robust improvements are possible for the least squares classifier.

\section{Method}
\label{section:method}

\subsection{Supervised Multivariate Least Squares Classification} \label{section:leastsquares}

Least squares classification \cite{Hastie2001,Rifkin2003} is the direct application of well-known ordinary least squares regression to a classification problem. A linear model is assumed and the parameters are minimized under squared loss. Let $\mathbf{X}$ be an $\Nlab \times (\featdim+1)$ design matrix with $\Nlab$ rows containing vectors of length equal to the number of features $\featdim$ plus a constant feature to encode the intercept. Vector $\textbf{y}$ denotes an $\Nlab \times 1$ vector of  class labels. We encode one class as $0$ and the other as $1$.  The multivariate version of the empirical risk function for least squares regression is given by
\begin{equation} \label{squaredloss}
\hat{R}(\boldsymbol{\beta}) = \frac{1}{n} \left\|  \mathbf{X} \boldsymbol{\beta}-\mathbf{y} \right\| _2^2
\end{equation}
The well known closed-form solution for this problem is found by setting the derivative with respect to $\boldsymbol{\beta}$ equal to $\textbf{0}$ and solving for $\boldsymbol{\beta}$, giving:
\begin{equation} \label{olssolution}
\boldsymbol{\hat{\beta}}=\left(\mathbf{X}^T \mathbf{X}\right)^{-1} \mathbf{X}^T \mathbf{y}
\end{equation}
In case $\textbf{X}^T \textbf{X}$ is not invertible (for instance when $n<(\featdim+1)$), a pseudo-inverse is applied. As we will see, the closed form solution to this problem will enable us to formulate our semi-supervised learning approach in terms of a standard quadratic programming problem, which is easy to optimize.

\subsection{Implicitly Constrained Least Squares Classification} \label{section:icls}

In the semi-supervised setting, apart from a design matrix $\textbf{X}$ and target vector $\textbf{y}$, an additional set of measurements $\textbf{X}_u$ of size $\Nunl \times (\featdim+1)$ \emph{without} a corresponding target vector $\textbf{y}_u$ is given. In what follows, $\mathbf{X}_e=\begin{bmatrix} \mathbf{X}^T  & \mathbf{X}_u^T \end{bmatrix}^T$ denotes the extended design matrix which is simply the concatenation of the design matrices of the labeled and unlabeled objects.

In the implicitly constrained approach, we propose that a sensible solution to incorporate the additional information from the unlabeled objects is to search within the set of classifiers that can be obtained by all possible labelings $\textbf{y}_u$, for the one classifier that minimizes the \emph{supervised} empirical risk function \eqref{squaredloss}. This set, $\mathcal{C}_{\boldsymbol{\beta}}$, is formed by the $\boldsymbol{\beta}$'s that would follow from training supervised classifiers on all (labeled and unlabeled) objects going through all possible soft labelings for the unlabeled samples, i.e., using all $\textbf{y}_u \in [0,1]^{\Nunl}$. Since these supervised solutions have a closed form, this can be written as:
\begin{equation} \label{constrainedregion}
\mathcal{C}_{\boldsymbol{\beta}} := \left\{   \boldsymbol{\beta} = \left( {\XeT} {\Xe} \right)^{-1} {\XeT} \ye: \mathbf{y}_u \in [0,1]^{\Nunl} \right\}
\end{equation}
This constrained region $\mathcal{C}_{\boldsymbol{\beta}}$, combined with the supervised loss that we want to optimize in equation \eqref{squaredloss}, gives the following definition for implicitly constrained semi-supervised least squares classification:
\begin{equation}
\begin{aligned}
&\operatorname*{argmin}_{\boldsymbol{\beta} \in \mathbb{R}^{\featdim+1}} & \frac{1}{n}  ||\mathbf{X} \boldsymbol{\beta}-\mathbf{y}||^2  \\
& \text{subject to} & \boldsymbol{\beta} \in \mathcal{C}_{\boldsymbol{\beta}}  \\
\end{aligned}
\end{equation}
Since $\boldsymbol{\beta}$ is fixed for a particular choice of $\textbf{y}_u$ and has a closed form solution, we can rewrite the minimization problem in terms of $\textbf{y}_u$ instead of $\boldsymbol{\beta}$:
\begin{equation} \label{icls}
\begin{aligned}
& \operatorname*{argmin}_{\mathbf{y}_u} & \frac{1}{n}  \left\|  \X \left(\XeT \Xe \right)^{-1} \XeT \ye - \mathbf{y} \right\|_2^2 \\ 
& \text{subject to} & \mathbf{y}_u \in [0,1]^{\Nunl} \\
\end{aligned}
\end{equation}
Solving for $\mathbf{y}_u$ gives a labeling that we can use to construct the semi-supervised classifier using equation \eqref{olssolution} by considering the imputed labels as the labels for the unlabeled data. The problem defined in equation \eqref{icls}, is a standard quadratic programming problem. Due to the simple box constraints on the unknown labels this can be solved efficiently using a quasi-Newton approach that takes into account the simple $[0,1]$ bounds, such as L-BFGS-B \cite{Byrd1995}.

\section{Theoretical Results}
\label{section:theoreticalresults}
We will examine this procedure by considering it in a simple, yet illustrative setting. In this case we will, in fact, prove this procedure will \emph{never} give worse performance than the supervised solution.
Consider the case where we have just one feature $x$, a limited set of labeled instances and assume we know the probability density function of this feature $f_X(x)$ exactly. 
This last assumption is similar to having unlimited unlabeled data. 
We consider a linear model with no intercept: $y = x \beta$ where $y$ is set as $0$ for one class and $1$ for the other. 
For new data points, estimates $\hat{y}$ can be used to determine the predicted label of an object by using a threshold set at, for instance, $0.5$.

The expected squared loss, or risk, for this model is given by:

\begin{equation} \label{eq:trueloss}
R^*(\beta) = \sum_{y \in \{0,1\}}{ \int_{-\infty}^{\infty}(x \beta - y)^2  f_{X,Y}(x,y)  \mathrm{d}x}
\end{equation}
Where $f_{X,Y}=P(y|x) f_X(x)$. We will refer to this as the joint density of X and Y. Note, however, that this is not strictly a density, since it deals with the joint distribution over a continuous $X$ and a discrete $Y$. The optimal solution $\beta^*$ is given by the $\beta$ that minimizes this risk:

\begin{equation} \label{eq:bayesoptimal}
\beta^* = \operatorname*{argmin}_{\beta \in \mathbb{R}} R^*(\beta)
\end{equation}
We will show the following result:
\begin{theorem}
\label{theorem:1d}
Given a linear model without intercept, $y = x\beta$, and $f_X(x)$ known, the estimate obtained through implicitly constrained least squares always has an equal or lower risk than the supervised solution: $$R^\ast (\hat{\beta}_{semi}) \le R^\ast (\hat{\beta}_{sup})$$
\end{theorem}

\begin{proof}

Setting the derivative of \eqref{eq:trueloss} with respect to $\beta$ to $0$ and rearranging we get:

\begin{eqnarray}
&\beta & = \left( \int_{-\infty}^{\infty} { x^2 f_X(x) \mathrm{d}x} \right)^{-1} \sum_{y \in \{0,1\}} \int_{-\infty}^{\infty} { x y f_{X,Y}(x,y) \mathrm{d}x } \\
& & =    \left( \int_{-\infty}^{\infty} { x^2 f_X(x) \mathrm{d}x} \right)^{-1}  \int_{-\infty}^{\infty} { x f_X(x) \sum_{y \in \{0,1\}} y P(y|x) \mathrm{d}x} \\
& & =   \left( \int_{-\infty}^{\infty} { x^2 f_X(x) \mathrm{d}x} \right)^{-1}  \int_{-\infty}^{\infty} { x f_X(x) \mathbb{E}[y|x] \mathrm{d}x} \label{eqn:sslsolution}
\end{eqnarray}

In this last equation, since we assume $f_X(x)$ as given, the only unknown is the function $\mathbb{E}[y|x]$, the expectation of the label $y$, given $x$. Now suppose we consider every possible labeling of the unlimited number of unlabeled objects including fractional labels, that is, every possible function where $\mathbb{E}[y|x] \in [0,1]$. Given this restriction on $\mathbb{E}[y|x]$, the second integral in \eqref{eqn:sslsolution} becomes a re-weighted version of the expectation operation $\mathbb{E}[x]$. By changing the choice of $\mathbb{E}[y|x]$ one can vary the value of this integral, but it will always be bounded on an interval on $\mathbb{R}$. It follows that all possible $\beta$s also form an interval on $\mathbb{R}$, which we will refer to as the constrained set $\mathcal{C}_{\boldsymbol{\beta}}$. The optimal solution has to be in this interval, since it corresponds to a particular but unknown labeling $\mathbb{E}[y|x]$. Note from \eqref{eqn:sslsolution} that the boundaries of this interval are typically finite, unless the second moment of $X$ is equal to $0$.

\begin{figure}[t] 
  \centering
      \includegraphics[width=0.6\textwidth]{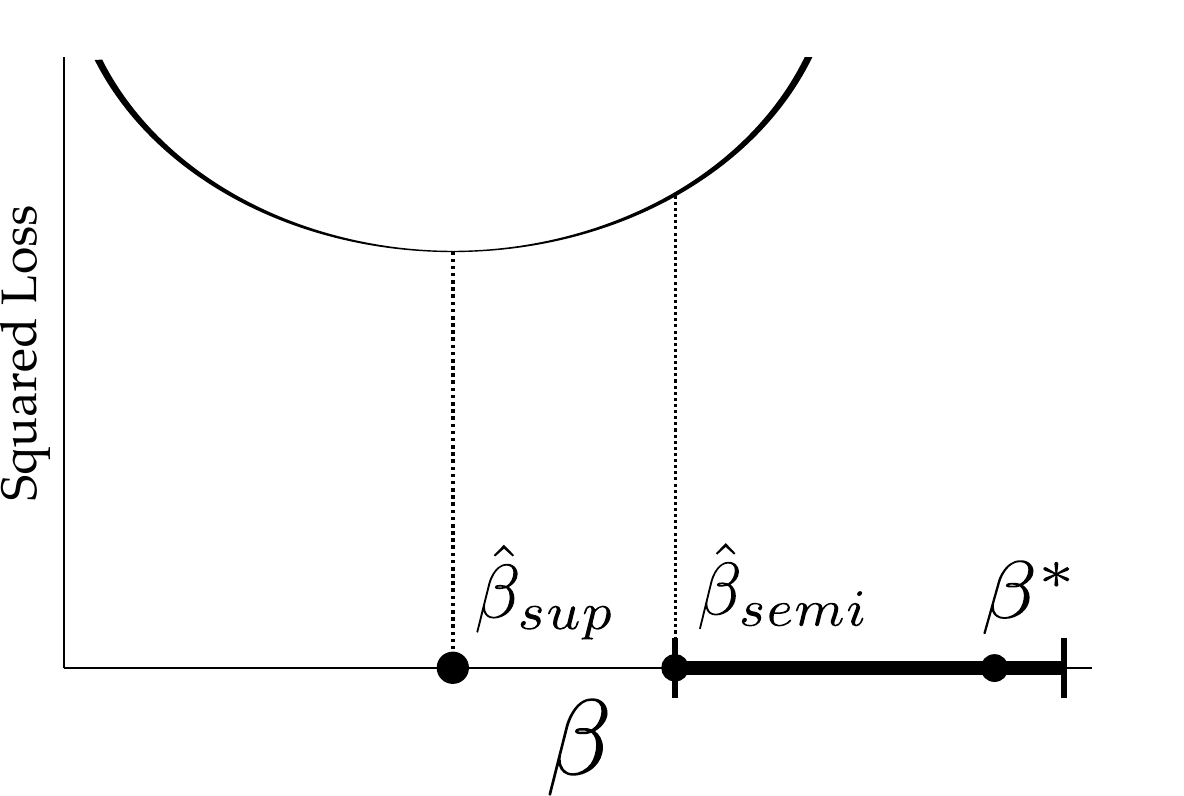}
  \caption{An example where implicitly constrained optimization improves performance. The supervised solution $\hat{\beta}_{sup}$ which minimizes the supervised loss (the solid curve), is not part of the interval of allowed solutions. The solution that minimizes this supervised loss within the allowed interval is $\hat{\beta}_{semi}$. This solution is closer to the optimal solution ${\beta}^{\ast}$ than the supervised solution $\hat{\beta}_{sup}$.} \label{fig:constrainedproblem}
\end{figure}

Using the set of labeled data, we can construct a supervised solution $\hat{\beta}_{sup}$ that minimizes the loss on the training set of $\Nlab$ labeled objects, see Figure \ref{fig:constrainedproblem}:

\begin{equation} \label{supervisedsolution}
\hat{\beta}_{sup} = \operatorname*{argmin}_{\beta \in \mathbb{R}} \sum_{i=1}^{\Nlab} (x_i \beta - y_i)^2
\end{equation}

Now, either this solution falls within the constrained region, $\hat{\beta}_{sup} \in \mathcal{C}_{\boldsymbol{\beta}}$ or not, $\hat{\beta}_{sup} \notin \mathcal{C}_{\boldsymbol{\beta}}$, with different consequences:

\begin{enumerate}
  \item If $\hat{\beta}_{sup} \in \mathcal{C}_{\boldsymbol{\beta}}$ there is a labeling of the unlabeled points that gives us the same value for $\beta$. Therefore, the solution falls within the allowed region and there is no reason to update our estimate. Therefore $\hat{\beta}_{semi}=\hat{\beta}_{sup}$.
  \item Alternatively, if $\hat{\beta}_{sup} \notin  \mathcal{C}_{\boldsymbol{\beta}}$, the solution is outside of the constrained region (as shown in Figure \ref{fig:constrainedproblem}): there is no possible labeling of the unlabeled data that will give the same solution as $\hat{\beta}_{sup}$. We then update the $\beta$ to be the $\beta$ within the constrained region that minimizes the loss on the supervised training set. As can be seen from Figure \ref{fig:constrainedproblem}, this will be a point on the boundary of the interval. Note that $\hat{\beta}_{semi}$ is now closer to $\beta^{*}$ than $\hat{\beta}_{sup}$. Since the true loss function $R^*(\beta)$ is convex  and achieves its minimum in the optimal solution, corresponding to the true labeling, the risk of our semi-supervised solution will always be equal to or lower than the loss of the supervised solution.
\end{enumerate}

Thus, the proposed update either improves the estimate of the parameter $\beta$ or it does not change the supervised estimate. In no case will the semi-supervised solution be worse than the supervised solution, in terms of the expected squared loss.

\end{proof}

\section{Empirical Results} 
\label{section:empiricalresults}


Since we extended the least squares classifier to the semi-supervised setting, we compare how, for different sizes of the unlabeled sample,  our semi-supervised least squares approach fares against supervised least squares classification without the constraints. For comparison we included an alternative semi-supervised approach by applying self-learning to the least squares classifier. In self-learning \cite{McLachlan1975}, the supervised classifier is updated iteratively by using its class predictions on the unlabeled objects as the labels for the unlabeled objects in the next iteration. This is done until convergence.


A description of the datasets used for our experiments is given in Table \ref{table:datasets}. We use datasets from both the UCI repository \cite{Bache2013} and from the benchmark datasets proposed by \cite{Chapelle2006}. While the benchmark datasets proposed in \cite{Chapelle2006} are useful, in our experience, the results on these datasets are very homogeneous because of the similarity in the dimensionality and their low Bayes errors. The UCI datasets are more diverse both in terms of the number of objects and features as well as the nature of the underlying problems. Taken together, this collection allows us to investigate the properties of our approach for a wide range of problems.

\begin{table}[t] 
\caption{Description of the datasets used in the experiments. Features indicates the dimensionality of the design matrix after categorical features are expanded into dummy variables.}
\begin{center}
\begin{tabular}{rrrr}
  \hline
 Name & \hspace{5 mm} \# Objects & \hspace{5 mm} \# Features  & \hspace{5 mm} Source \\ 
  \hline
  \texttt{Ionosphere} & 351 &  33 & \cite{Bache2013} \\ 
  \texttt{Parkinsons} & 195 &  22 & \cite{Bache2013} \\ 
  \texttt{Diabetes} & 768 &   8 & \cite{Bache2013} \\ 
  \texttt{Sonar} & 208 &  60 & \cite{Bache2013} \\ 
  \texttt{SPECT} & 267 &  22 & \cite{Bache2013} \\ 
  \texttt{SPECTF} & 267 &  44 & \cite{Bache2013} \\ 
  \texttt{WDBC} & 569 &  30 & \cite{Bache2013} \\
  \texttt{Digit1} & 1500 & 241 & \cite{Chapelle2006} \\ 
  \texttt{USPS} & 1500 & 241 & \cite{Chapelle2006}  \\ 
  \texttt{COIL2} & 1500 & 241 & \cite{Chapelle2006} \\ 
  \texttt{BCI} & 400 & 118 & \cite{Chapelle2006} \\ 
  \texttt{g241d} & 1500 & 241 & \cite{Chapelle2006} \\ 
   \hline
\end{tabular}
\end{center}

\label{table:datasets}
\end{table}

\subsection{Comparison of Learning Curves}
We study the behavior of the expected classification error of the ICLS procedure for different sizes for the unlabeled set. This statistic has two desired properties. First of all it should never be higher than the expected classification error of the supervised solution, which is based on only the labeled data. Secondly, the expected classification error should not increase as we add more unlabeled data. 

Experiments were conducted as follows. For each dataset, $\Nlab$ labeled points were randomly chosen, where we make sure it contains at least 1 object from each of the two classes.  With fewer than $\featdim$ samples, the \emph{supervised} least squares classifier is known to deteriorate in performance as more data is added, a behavior known as peaking \cite{Opper1996,Raudys1998}. Since this is not the topic of this work, we will only consider the situation in which the labeled design matrix is of full rank, which we ensure by setting $\Nlab=\featdim+5$, the dimensionality and intercept of the dataset plus five observations. For all datasets we ensure a minimum of $\Nlab=20$ labeled objects.

Next, we create unlabeled subsets of increasing size $\Nunl=[2,4,8,...,1024]$ by randomly selecting points from the original dataset without replacement. The classifiers are trained using these subsets and the classification performance is evaluated on the remaining objects. Since the test set decreases in size as the number of unlabeled objects increases, the standard error slightly increases with the number of unlabeled objects.

This procedure of sampling labeled and unlabeled points is repeated $100$ times. The results of these experiments are shown in Figure \ref{fig:learningcurves1}. We report the mean classification error as well as the standard error of this mean. As can be seen from the tight confidence bands, this offers an accurate estimate of the expected classification error.


\begin{figure}[t] 
  \centering
      \hspace*{-3.0cm}
      \input{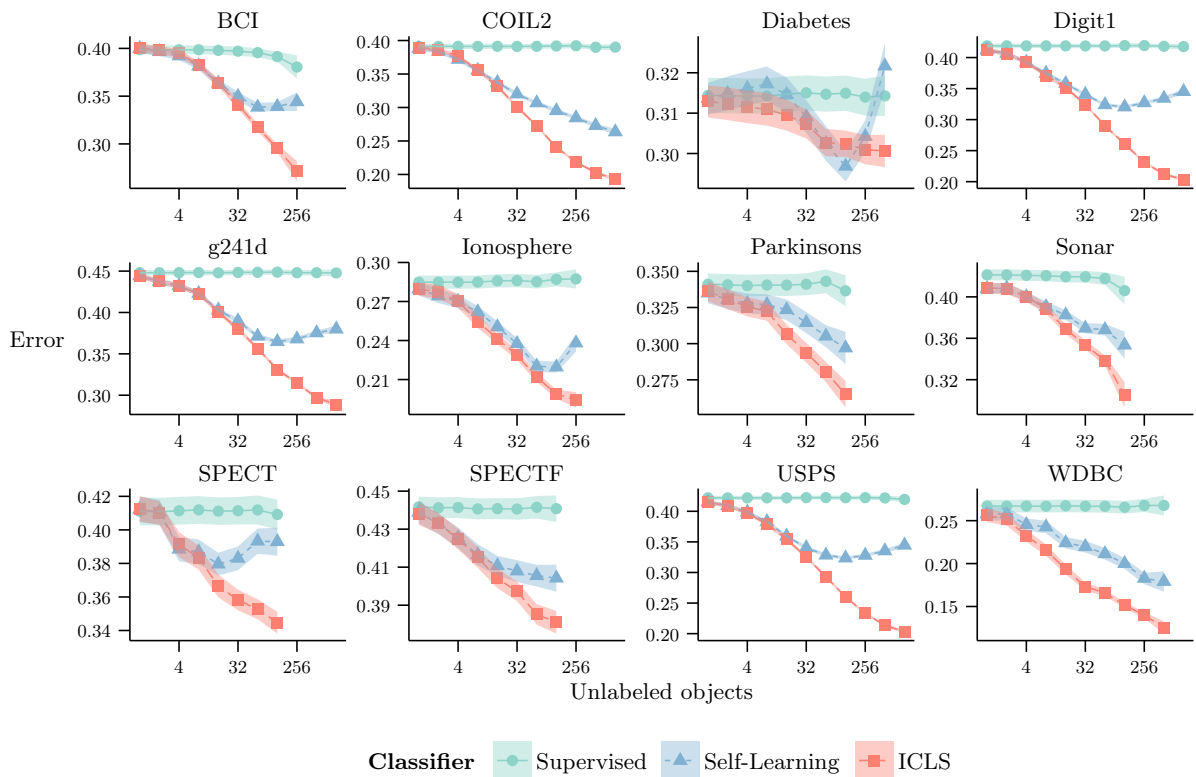}
  \caption{Mean classification error for $\Nlab=\max(\featdim+5,20)$ and $100$ repeats. The shaded areas indicate $+/-$ the standard error of the mean.} \label{fig:learningcurves1}
\end{figure}

We find that, generally, the ICLS procedure has monotonically decreasing error curves as the number of unlabeled samples increases, unlike self-learning. On the \texttt{Diabetes} dataset, the performance of self-learning becomes worse than the supervised solution when more unlabeled data is added, while the ICLS classifier again exhibits a monotonic decrease of the average error rate. 


\subsection{Benchmark performance}
We now consider the performance of these classifiers in a cross-validation setting. This experiment is set up as follows. For each dataset, the objects are randomly divided into $10$ folds. We iteratively go through the folds using $1$ fold as validation set, and the other $9$ as the training set. From this training set, we then randomly select $\Nlab=\featdim+5$ labeled objects, as in the previous experiment, and use the rest as unlabeled data. After predicting labels for the validation set for each fold, the classification error is then determined by comparing the predicted labels to the real labels. This is repeated $20$ times, while randomly assigning objects to folds in each iteration.


\begin{table}[t]
\caption{Average 10-fold cross-validation error and number of times the error of the semi-supervised classifier is higher than the supervised error for 20 repeats. Oracle refers to the performance of the least squares classifier trained when all labels are known. Indicated in $\mathbf{bold}$ is when a semi-supervised classifier has significantly lower error than the other, using a Wilcoxon signed rank test at $0.01$ significance level. A similar test is done to determine whether a semi-supervised classifier is significantly worse than the supervised classifier, indicated by \underline{underlined} values.} \label{table:cvresults}
\centering
\begin{tabular}{lllll}
  \hline
Dataset & Supervised \hspace{0.3cm} & Self-Learning \hspace{0.3cm} & ICLS \hspace{01.3cm} & Oracle \\ 
  \hline
\texttt{ Ionosphere } & 0.29 & 0.24 (1) & \textbf{0.19 (0)} & 0.13 \\ 
  \texttt{ Parkinsons } & 0.33 & 0.29 (3) & 0.27 (0) & 0.11 \\ 
  \texttt{ Diabetes } & 0.32 & \underline{0.33 (16)} & \textbf{0.31 (2)} & 0.23 \\ 
  \texttt{ Sonar } & 0.42 & 0.37 (1) & \textbf{0.32 (1)} & 0.25 \\ 
  \texttt{ SPECT } & 0.42 & 0.40 (7) & \textbf{0.33 (0)} & 0.17 \\ 
  \texttt{ SPECTF } & 0.44 & 0.41 (3) & \textbf{0.36 (0)} & 0.22 \\ 
  \texttt{ WDBC } & 0.27 & 0.17 (0) & \textbf{0.12 (0)} & 0.04 \\ 
  \texttt{ Digit1 } & 0.41 & 0.34 (0) & \textbf{0.20 (0)} & 0.06 \\ 
  \texttt{ USPS } & 0.42 & 0.35 (0) & \textbf{0.20 (0)} & 0.09 \\ 
  \texttt{ COIL2 } & 0.40 & 0.27 (0) & \textbf{0.19 (0)} & 0.10 \\ 
  \texttt{ BCI } & 0.40 & 0.35 (0) & \textbf{0.28 (0)} & 0.16 \\ 
  \texttt{ g241d } & 0.45 & 0.39 (0) & \textbf{0.29 (0)} & 0.13 \\ 
   \hline
\end{tabular}
\end{table}

The results shown in Table \ref{table:cvresults} tell a similar story to those in the previous experiment. Most importantly for the purposes of this paper, ICLS, in general, offers solutions that give at least no higher expected classification error than the supervised procedure. Moreover, in most of the cross-validation repeats, the error is not higher than the supervised error, although it does occur in some instances. 

\section{Discussion}
The results presented in this paper are encouraging in the light of negative theoretical performance results in the semi-supervised literature \cite{Cozman2006}. The result in Theorem 1 indicates the proposed procedure is in some way robust against reduction in performance. The empirical results in the previous section indicate a similar result in terms of the expected classification error, at least on this collection of datasets. These empirical observations are interesting because the loss that was evaluated in these experiments is misclassification error and not the squared loss that was considered in Theorem 1. Furthermore the experiments were carried in the multivariate setting with an intercept term using limited unlabeled data, rather than the unlimited unlabeled data setting considered in the theorem. This indicates that minimizing the supervised loss over the subset $\Cb$, leads to a semi-supervised learner with desirable behavior, both theoretically in terms of risk and empirically in terms of classification error.

It has been argued that, for discriminative classifiers, semi-supervised learning is impossible without additional assumptions about the link between labeled and unlabeled objects \cite{Seeger2001,Singh2008}. ICLS, however, is both a discriminative classifier and no explicit additional assumptions about this link are made. Any assumptions that are present follow, implicitly, from the choice of squared loss as the loss function and from the chosen hypothesis space. One could argue that constraining the solutions to $\Cb$ is an assumption as well. While this is true, it corresponds to a very weak assumption about the supervised classifier: that it will improve when we add additional labeled data. The lack of additional assumptions has another advantage: no additional parameters need to be correctly set for the results in sections 4 and 5 to hold. There is, for instance, no parameter to be chosen for the importance of the unlabeled data. Therefore, implicitly constrained semi-supervised learning is a very different approach to semi-supervised learning than current alternatives.

An open question is what other classifiers could benefit from the implicitly constrained approach considered here. Using negative log likelihood as a loss function, for instance, also leads to interesting semi-supervised classifiers, for instance in linear discriminant analysis \cite{Krijthe2014}. For other classifiers, the definition of the constraints used in this work might not lead to any useful constraints at all such that the supervised solution is always recovered. One would have to define additional constraints on the solutions in $\Cb$. The minimization of the supervised loss, considered in this paper, could still be relevant in these cases to construct a semi-supervised classifier that has similar robustness against deterioration in performance as ICLS.

\section{Conclusion}
This contribution introduced a new semi-supervised approach to least squares classification. By implicitly considering all possible labelings of the unlabeled objects and choosing the one that minimizes the loss on the labeled observations, we derived a robust classifier with a simple quadratic programming formulation. For this procedure, in the univariate setting with a linear model without intercept, we can prove it never degrades performance in terms of squared loss (Theorem 1). Experimental results indicate that in expectation this robustness also holds in terms of classification error on real datasets. Hence, semi-supervised learning for least squares classification without additional assumptions can lead to improvements over supervised least squares classification both in theory and in practice.

\subsection*{Acknowledgments}\label{sec:Acknowledgments}
Part of this work was funded by project P23 of the Dutch public-private research community COMMIT.

\bibliographystyle{splncs03}
\bibliography{library}

\begin{thebibliography}{10}
\providecommand{\url}[1]{\texttt{#1}}
\providecommand{\urlprefix}{URL }

\bibitem{Bache2013}
Bache, K., Lichman, M.: {\{UCI\} Machine Learning Repository} (2013),
  \url{http://archive.ics.uci.edu/ml}

\bibitem{Bennett1998}
Bennett, K.P., Demiriz, A.: {Semi-supervised support vector machines}. In:
  Advances in Neural Information Processing Systems 11. pp. 368--374 (1998)

\bibitem{Bottou2010}
Bottou, L.: {Large-scale machine learning with stochastic gradient descent}.
  In: Proceedings of COMPSTAT'2010. pp. 177--186. Springer (2010)

\bibitem{Byrd1995}
Byrd, R.H., Lu, P., Nocedal, J., Zhu, C.: {A limited memory algorithm for bound
  constrained optimization}. SIAM Journal on Scientific Computing  16(5),
  1190--1208 (1995)

\bibitem{Chapelle2006}
Chapelle, O., Sch\"{o}lkopf, B., Zien, A.: {Semi-supervised learning}. MIT
  press (2006)

\bibitem{Cozman2006}
Cozman, F., Cohen, I.: {Risks of Semi-Supervised Learning}. In: Chapelle, O.,
  Sch\"{o}lkopf, B., Zien, A. (eds.) Semi-Supervised Learning, chap.~4, pp.
  56--72. MIT press (2006)

\bibitem{Cozman2003}
Cozman, F.G., Cohen, I., Cirelo, M.C.: {Semi-Supervised Learning of Mixture
  Models}. In: Proceedings of the Twentieth International Conference on Machine
  Learning (2003)

\bibitem{Hastie2001}
Hastie, T., Tibshirani, R., Friedman, J.H.: {The Elements of Statistical
  Learning}. Spinger (2001)

\bibitem{Krijthe2014}
Krijthe, J.H., Loog, M.: {Implicitly Constrained Semi-Supervised Linear
  Discriminant Analysis}. In: International Conference on Pattern Recognition.
  pp. 3762--3767. Stockholm (2014)

\bibitem{Li2015}
Li, Y.F., Zhou, Z.H.: {Towards Making Unlabeled Data Never Hurt}. IEEE
  Transactions on Pattern Analysis and Machine Intelligence  37(1),  175--188
  (Jan 2015)

\bibitem{Loog2014b}
Loog, M., Jensen, A.: Semi-supervised nearest mean classification through a
  constrained log-likelihood. IEEE Transactions on Neural Networks and Learning
  Systems  26(5),  995--1006 (May 2015)

\bibitem{Loog2014a}
Loog, M.: {Semi-supervised linear discriminant analysis through
  moment-constraint parameter estimation}. Pattern Recognition Letters  37,
  24--31 (Mar 2014)

\bibitem{McLachlan1975}
McLachlan, G.J.: {Iterative Reclassification Procedure for Constructing an
  Asymptotically Optimal Rule of Allocation in Discriminant Analysis}. Journal
  of the American Statistical Association  70(350),  365--369 (1975)

\bibitem{Nigam2000}
Nigam, K., McCallum, A.K., Thrun, S., Mitchell, T.: {Text classification from
  labeled and unlabeled documents using EM}. Machine learning  34,  1--34
  (2000)

\bibitem{Opper1996}
Opper, M., Kinzel, W.: {Statistical Mechanics of Generalization}. In: Domany,
  E., Hemmen, J.L., Schulten, K. (eds.) Models of Neural Networks III, pp.
  151--209. Springer, New York (1996)

\bibitem{Poggio2003}
Poggio, T., Smale, S.: {The Mathematics of Learning: Dealing with Data}.
  Notices of the AMS pp. 537--544 (2003)

\bibitem{Raudys1998}
Raudys, S., Duin, R.P.: {Expected classification error of the Fisher linear
  classifier with pseudo-inverse covariance matrix}. Pattern Recognition
  Letters  19(5-6),  385--392 (Apr 1998)

\bibitem{Rifkin2003}
Rifkin, R., Yeo, G., Poggio, T.: {Regularized least-squares classification}.
  Nato Science Series Sub Series III Computer and Systems Sciences 190  (2003)

\bibitem{Seeger2001}
Seeger, M.: {Learning with labeled and unlabeled data}. Tech. rep. (2001)

\bibitem{Singh2008}
Singh, A., Nowak, R.D., Zhu, X.: {Unlabeled data: Now it helps , now it
  doesn’t}. In: Advances in Neural Information Processing Systems. pp.
  1513--1520 (2008)

\bibitem{Tibshirani1996}
Tibshirani, R.: {Regression shrinkage and selection via the lasso}. Journal of
  the Royal Statistical Society. Series B  58(1),  267--288 (1996)

\bibitem{Widrow1960}
Widrow, B., Hoff, M.E.: {Adaptive switching circuits.} In: IRE WESCON
  Convention Record 4. pp. 96--104 (1960)

\bibitem{Zhu2009}
Zhu, X., Goldberg, A.B.: {Introduction to Semi-Supervised Learning}, vol.~3.
  Morgan \& Claypool (2009)

\end{thebibliography}

\end{document}